\documentclass[twocolumn]{article}
\usepackage[width=17cm,height=22cm]{geometry}
\usepackage[english]{babel}
\usepackage[utf8]{inputenc}
\usepackage{fancyvrb}
\usepackage{authblk}
\usepackage{hyperref}
\usepackage{amsmath,amsthm,amssymb} 
\usepackage{bbm}
\usepackage{natbib}
\usepackage{hyperref}
\usepackage{multirow}
\usepackage{graphicx}
\usepackage{float}
\usepackage{color,soul}
\DeclareMathAlphabet{\mathpzc}{OT1}{pzc}{m}{it}

\newtheorem{proposition}{Proposition}[section]

\makeatletter
\DeclareFontEncoding{LS1}{}{}
\DeclareFontSubstitution{LS1}{stix}{m}{n}
\DeclareMathAlphabet{\mathscr}{LS1}{stixscr}{m}{n}
\makeatother

\title{Similarities between policy gradient methods (PGM) in \\ reinforcement learning (RL) and supervised learning (SL)}
\author[1]{Eric Benhamou\thanks{eric.benhamou@dauphine.eu}}
\affil[1]{A.I Square Connect, Lamsade PSL}

\begin{document}
\maketitle

\begin{abstract}
Reinforcement learning (RL) is about sequential decision making and is traditionally opposed to supervised learning (SL) and unsupervised learning (USL). In RL, given the current state, the agent makes a decision that may influence the next state as opposed to SL (and USL) where, the next state remains the same, regardless of the decisions taken, either in batch or online learning.  Although this difference is fundamental between SL and RL, there are connections that have been overlooked. In particular, we prove in this paper that gradient policy method can be cast as a supervised learning problem where true label are replaced with discounted rewards. We provide a new proof of policy gradient methods (PGM) that emphasizes the tight link with the cross entropy and supervised learning. 
We provide a simple experiment where we interchange label and pseudo rewards. We conclude that other relationships with SL could be made if we modify the reward functions wisely.
\end{abstract}

\medskip

\noindent\textbf{keywords}: Policy gradient, Supervised learning, \\ Cross entropy, Kullback Leibler divergence, entropy.

\section{Introduction}
In RL, PGM are frequently used \cite{Williams_1992,Sutton_1999,Silver_2014,Lillicrap_2015}. They are traditionally opposed to value learning methods \cite{Sutton_1998,Watkins_1992}. 
PGM principle is very simple. Improve gradually the policy through gradient descent. PGM have two important concepts. It is a policy method as its name emphasizes. It is also a gradient descent method. Policy means we observe and act. Gradient descent methods uses the fact that the best move locally is along the gradient. As this moves is at first order, a learning rate needs to ensure policy improvement is not too large at each step. PGM have been popularized in REINFORCE \cite{Williams_1992} and in \cite{Sutton_1999} and have received wider attention with Actor Critic methods \cite{Konda_2003,Peters_2008} in particular when using deep PGM \cite{Mnih_2016} that combines policy and value methods. In addition, recently, it has been found that an entropy regularization term may fasten convergence \cite{Donoghue_2016,Nachum_2017,Schulman_2017}.

When looking in details in REINFORCE \cite{Williams_1992}, we can remark that the gradient term with respect to the policy can indeed be interpreted as the log term in the cross entropy in supervised learning. If in addition, we make the bridge between RL and SL, emphasizing that RL problem can be reformulated as a SL problem where true labels are changed by expected discounted future rewards, and estimated probabilities by policy probabilities, the link between RL and SL becomes obvious. In addition, leveraging the tight relationship between cross entropy and Kullback Leibler divergence, we can interpret the entropy regularization terms very naturally. This is precisely the objective of this short paper. Call attention to the tight connection between RL and SL to give theoretical justification of some of the techniques used in PGMs. 

The paper is organized as follows. In section \ref{sec:SL}, we recall the various choices of functional losses in SL and exhibit that cross entropy is one of the main possibilities for loss functions. We flaunt the relationship between cross entropy and Kullback Leibler divergence, harping on the additional entropy term. In section \ref{sec:RL}, we present PGMs. We rub in the interpretation of RL problems as a modified cross entropy SL problem. We conclude in section \ref{sec:experiments} with a financial numerical experience using deep PGM stressing that in the specific case of action that do not influence the environment, the difference between RL and SL may be very tenuous.

\section{Related Work}
Looking at similarities and synergies between RL and SL together was for a long time overlooked. This can be easily explained as the two research communities were different and thought they were labouring on incompatible or at least very different approaches. However, there has been one type of learning that promotes the similitude between RL and SL. This has been Imitation Learning.

Imitation learning \citep{Schaal_1996} is a classic technique for learning from human demonstration. Imitation learning uses a supervised approach to imitate an expert's behaviors,  hence doing a RL task to accomplish a SL one. DAGGER~\citep{Ross_2011} is considered to be the mainstream imitation algorithm. It requests an action from the expert at each step. It uses an action sampled from a mixed distribution from the agent and the expert. It combines the observed states and demonstrated actions to train the agent successively. This has led to numerous extensions of this algorithm and in particular to deep version of it. Deeply AggreVaTeD~\citep{Sun_2017} extends DAGGER with deep neural networks and continuous action spaces.
Other approaches have been to stop opposing RL and SL and rather leverage SL in deep RL. \cite{Hester_2017} for instance have used SL tasks to train the agent better and created the method of Deep Q-learning from Demonstrations (DQfD). It tackles the problem of the necessity of huge amount of data for deep reinforcement learning (Deep RL) and the poor performance of Deep RL algorithms during initial phase of learning. The method of Deep Q-learning from Demonstrations (DQfD) solves the issue by combining RL techniques (temporal difference updates) with SL techniques (supervised classification of the demonstrator's actions) as the target network in Deep Q learning is initially trained with supervised learning. 

All of these works show that RL and SL are not as opposed as one may have thought. We argue here that as nice as these works are, they do not emphasizes that, ignoring for a while the issue of feedback effect of action on next state environment, a PGM can be reformulated as a SL task where true labels are changed into future expected reward while the PG can be interpreted as a cross entropy loss minimization.

\section{Supervised Learning} \label{sec:SL}
The goal of SL classification is to infer a function from labelled training data that maps inputs into labelled outputs. In psychology, this is sometimes analyzed as concept learning. 
The deduction of the function parameters is done traditionally through the optimization of a loss function. SL classifiers parameters are the ones of the optimal solution of the optimization program. To keep things simple, let us assume that we are looking at a binary classification problem. Let us assume we observe $D_n = \{(X_1,Y_1), \ldots, (X_n,Y_n)\}$ that are $n$ independent random copies of $(X,Y) \in \mathcal{X} \times \mathcal{Y}$. The feature $X$ lives in some abstract space $\mathcal{X} $ ($\mathbf{R}^d$ for instance) 
and $Y$ is called label. Binary classification assumes that $Y$ take two different values: $\mathcal{Y} = \{-1, 1\}$, 
while multi-class classification assumes $\mathcal{Y} = \{1, \ldots, K \}$. To keep things simple, we will only look at binary classification. Naturally, one would like to find a function 
$f: \mathcal{X} \mapsto \mathbf{R}$ that best maps $X$ to $Y$. 
We are also given a loss function $\mathscr{l}: \{-1,1\} \times  \{-1,1\}  \mapsto \mathbf{R}$ 
that measures the error of a specific prediction. The loss function value at an arbitrary point
$(Y, \hat{Y})$ reads as the cost incurred when predicting $\hat Y$ while true label is $Y$. 
 In classification the loss function is often a zero-one loss, that is, $\mathscr{l}(Y,\hat Y)$ 
is zero when the predicted label matches the true label $Y=\hat Y$ and one otherwise. 
To find our best classifier, we look for the classifier with the smallest expected loss. 
In other words, we look up for the function $f$ that  minimizes the expected $\mathscr{l}$-risk,
 given by $\mathcal{R}_{\mathscr{l}}(h)= \mathbf{E}_{X \times Y}[\mathscr{l}(Y,f(X))]$.

Another naive approach is to minimize the empirical classification error $\mathbb{E}[ \mathbbm{1}_{\{-Y f(X) \geq 0\}}]$. 
To bypass the non convexity of $\mathbbm{1}_{\mathbb{R}_{+}}$, we use convex risk minimization (CRM) \cite{Boucheron_2003}. CRM defines a convex surrogate for the classification problem, called the cost function $\varphi:  \mathbb{R} \mapsto \mathbb{R}_{+}$ convex, non-decreasing 
such that $\varphi(0) = 1$, hence $\varphi \geq \mathbbm{1}_{\mathbb{R}_{+}}$. 
The classification problem consists  in minimizing the expected $\varphi$-risk : $\mathbb{E}[ \varphi(- Y f(X)) ]$.
Typical loss functions are
\begin{itemize}
\item square loss: $\varphi(u) = (1 + u)^2$ for $u \geq 0$ and $\varphi(u)=1$ for $u \leq 0$, leading to regression methods. 
\item perplexity loss: $\varphi(u) =\log(e (1 +u )))$ for $u > -1$ and $\varphi(u)=-\infty$ for $u \leq -1$. 
\item logit loss: $\varphi(u) =\log_2( 1 + e^u )$, leading to logistic regression methods and cross entropy.
\item hinge loss: $\varphi(u) = max(0,1+u)$. This leads to SVM methods.
\item exponential loss: $\varphi(u) = e^u$. 
\end{itemize}

We have also loss function defined directly between $Y$ and $\hat{Y}$ as follows:
\begin{itemize}
\item mean square error: $\mathscr{l}(Y, f(X)) = (Y-f(X))^2$ that leads to standard regression methods. 
\item mean absolute error: $\mathscr{l}(Y, f(X)) = |Y-f(X)|$.
\item cross entropy: $\mathscr{l}(Y, f(X)) = - \frac{1+Y}{2} \log \frac{1+f(X)}{2}$  leading to logistic regression with the usual convention  $0 \, \log 0 \! = \! 0$ justified by $\lim\limits_{x \to 0^+} x \log x = 0$ (trivially obtained by L'Hopital's rule).
\item Huber loss: $\mathscr{l}(Y, f(X)) = \frac 1 2 (Y-f(X))^2$ if $|Y-f(X)| \leq \delta$ and $\mathscr{l}(Y, f(X)) = \delta (  |Y-f(X)| -  \frac 1 2  \delta^2$
\end{itemize}

We see from above that for SL there are numerous loss functions and that cross entropy is one criterium among others. We will see that this cross entropy has a nice interpretation in RL.

\section{Reinforcement Learning Background}\label{sec:RL}
RL is usually modeled by an agent that interacts with an environment $\mathcal{E}$ over a number of discrete time steps.
At each time step $t$, the agent levies a state $s_t$ and picks an action $a_t$ from a set of possible actions $\mathcal{A}$. This choice is done according to its policy $\pi$, where $\pi$ is a mapping from states $s_t$ to actions $a_t$.
Once the action is decided and executed, the agent levies the next state $s_{t+1}$ and a scalar reward $r_t$.
The goes on until the agent reaches a terminal state.
The expected cumulated discounted return $R_t = \sum_{k=0}^{\infty} \gamma^k r_{t+k}$ is the sum of accumulated returns, where at each  time step, future returns are discounted with the discount factor $\gamma \in (0,1]$.
At time $t$, a rational agent seeks to maximize its expected return given his current state $s_t$.

We traditionally define
\begin{itemize}
\item the value of state $s$ under policy $\pi$ is defined as $V^{\pi}(s) = \mathbb{E}\left[R_t|s_t=s\right]$ and is simply the expected return for following policy $\pi$ from state $s$ (\citep{Watkins_1989}).
\item the action value function  under policy $\pi$  $Q^{\pi}(s,a) =$ $\mathbb{E}\left[R_t|s_t=s, a\right]$ is defined as the expected return for selecting action $a$ in state $s$ and following policy $\pi$ (~\cite{Williams_1992}).
\end{itemize}
Both the optimal value function $Q^*(s,a) =$ $\max_{\pi} Q^{\pi}(s,a)$ 
and the optimal value of state $V^{*}(s) =\max_{\pi} V^{\pi}(s)$ satisfy  Bellmann equations. 

Whenever states and action are too large, we are forced to represent the action value function with a function approximator, such as a neural network.
Denoting the parameters $\theta$, the state action function writes $Q(s,a;\theta)$ 

The updates to $\theta$ can be derived from a variety of reinforcement learning algorithms. In particular, in value-based methods, policy-based model-free methods directly parameterize the policy $\pi(a|s;\theta)$ and update the parameters $\theta$ by performing, typically approximate, gradient ascent on $\mathbb{E}[R_t]$.

An illustration of such a method is REINFORCE due to ~\cite{Williams_1992}.
Standard REINFORCE updates the policy parameters $\theta$ in the direction $R_t\nabla_{\theta}$$\log\pi(a_t|s_t;\theta)$, which is an unbiased estimate of $\nabla_{\theta} \mathbb{E}[R_t]$.
It is possible to reduce the variance of this estimate while keeping it unbiased by subtracting a learned function of the state $b_t(s_t)$, known as a baseline~\citep{Williams_1992}, from the return.
The resulting gradient is
$\nabla_{\theta}\log\pi(a_t|s_t;\theta) \left(R_t-b_t(s_t)\right)$.
This approach can be viewed as an actor-critic architecture where the policy $\pi$ is the actor and the baseline $b_t$ is the critic\citep{Sutton_1998,Degris_2012}.

We now prove a new formulation of REINFORCE.
\begin{proposition}\label{prop:prop1}
The gradient descent in REINFORCE can also be computed by minimizing the following quantity
\begin{equation}
\widetilde{J}(\theta) = \lim_{N \to \infty} \frac{1}{N} \sum_{i=1}^N \sum_{t=1}^T R_{i}(\tau)\log{\pi_\theta}(a_{i, t} \mid s_{i,t})  
\end{equation}
For Advantage Actor Critic method, the gradient descent can also be computed by minimizing the following quantity:
\begin{equation}
\widetilde{J}(\theta) = \frac{\sum\limits_{i=1}^N \sum\limits_{t=1}^T A(s_{i,t},a_{i, t})\log{\pi_\theta}(a_{i, t} \mid s_{i,t}) }{N}
\end{equation}
\end{proposition}
\begin{proof}
See supplementary materials \ref{proof:prop1}
\end{proof}

It is enlightening to see that the two formulations, traditional reinforce and actor critic are very close to SL method with cross entropy. Recall that cross entropy is given by $Y \log(\hat{Y})$ for labels with value in $\{0, 1\}$. This leads to the tables \ref{table1} and \ref{table2}.

\begin{table}[!htbp]
\resizebox{0.47\textwidth}{!}{
\begin{tabular}{|l|l|l|}
\hline
Term 		& SL 	& RL  (REINFORCE)\\ \hline
\multirow{2.5}{1.5cm}{true label}	
& \multirow{2.5}{1cm}{$Y$}	
& \multirow{2.5}{3.5 cm}{expected future rewards: $R(\tau)$}  \\ [20pt] \hline
\multirow{2.5}{1.5cm}{log term}
&  \multirow{2.5}{1cm}{$\log( \hat Y)$}
&  \multirow{2.5}{3.5 cm}{log of policy: $\log{\pi_\theta}(a_{i, t} \mid s_{i,t})$} \\ [20pt]  \hline
\multirow{3}{1cm}{cross entropy} 
& \multirow{3}{1.5 cm}{$\frac{\sum\limits_{i=1}^N Y_{i}\log{\hat{Y}_i } }{N} $}
&  \multirow{2}{4cm}{Monte Carlo expectation: $\frac{\sum\limits_{i=1}^N \sum\limits_{t=1}^T R_{i}(\tau)\log{\pi_\theta}(a_{i, t} \mid s_{i,t}) }{N}$  } \\	 [30pt] \hline
\end{tabular}}
\caption{Comparing SL and RL for REINFORCE}\label{table1}
\end{table}

\begin{table}[!htbp]
\resizebox{0.47\textwidth}{!}{
\begin{tabular}{|l|l|l|}
\hline
Term 		& SL 				& RL  (A2C)\\ \hline
\multirow{2.5}{1.5cm}{true label}	
& \multirow{2.5}{1cm}{$Y$}	
& \multirow{2.5}{4cm}{expected advantage: $A(s,a)=Q(s,a)-V(s)$}  \\ [20pt] \hline

\multirow{2.5}{1.5cm}{log term}
&  \multirow{2.5}{1cm}{$\log( \hat Y)$}
&  \multirow{2.5}{3.5 cm}{log of policy: $\log{\pi_\theta}(a_{i, t} \mid s_{i,t})$} \\ [20pt]  \hline

\multirow{3}{1cm}{cross entropy} 
& \multirow{3}{1.5 cm}{$\frac{\sum\limits_{i=1}^N Y_{i}\log{\hat{Y}_i } }{N} $}
&  \multirow{2}{4.1cm}{Monte Carlo expectation: $\frac{\sum\limits_{i=1}^N \sum\limits_{t=1}^T A(s_{i,t},a_{i, t})\log{\pi_\theta}(a_{i, t} \mid s_{i,t}) }{N}$  }							 \\ [30pt] \hline
\end{tabular}}
\caption{Comparing SL and RL for AAC methods}\label{table2}
\end{table}
Last but not least, recall that there is a connection between cross entropy and Kullback Leibler divergence. Recall that the cross entropy for the distributions $p$ and $q$ over a given set is defined as follows:
\begin{equation}
H(p,q)=\mathbb{E} _{p}[-\log q]
\end{equation}
There is a straightforward connection to the Kullback–Leibler divergence 
$D_{\mathrm {KL} }(p\|q)$ of $q$ from $p$, sometimes 
referred to as the relative entropy of $p$ with respect to $q$ given by
\begin{equation}
H(p,q)=H(p) + D_{\mathrm {KL} }(p\|q)
\end{equation}
where $H(p)$ is the entropy of $p$. As the entropy term $H(p)$ is constant given the true distribution $p$, minimizing the cross entropy or the Kullback Leibler divergence is equivalent. However, for RL, this gives another nice interpretation. As shown previously, PGM can be cast as a cross entropy minimization program. Since the difference between Kullback Leibler divergence and cross entropy is this entropy term, it makes sense to incorporate this entropy term in our gradient descent optimization. In theory, the entropy term should be multiply by one . However in practice as the entropy term is estimated by the empirical entropy, it makes sense somehow to multiply the entropy term by a regularization term $\lambda$, leading to a variation of PGMs where instead of computing a gradient on cross entropy as in REINFORCE, we add an additional term called an entropy regularization. Somehow, this changes the minimization problem to something that is a modified Kullback Leibler divergence. The cross entropy term itself is computed on future expected reward. The analogy holds between SL and RL as long as actions do not influence or slightly influence the environment.

\section{Experiments}\label{sec:experiments} 
We will apply our remark to a very specific environment where actions do not influence the environment.
The considered reinforcement problem is a financial trading game concerning the 
Facebook stock (data were retrieved from \href{https://finance.yahoo.com/quote/FB}{https://finance.yahoo.com/quote/FB}. 
We denote by $(P_t)_{t=1, \ldots }$ the daily closing price of the Facebook stock in sequential order. 
For each day, we compute the daily return as follows $r_t= \frac{P_t}{P_{t-1}}-1$

The environment is composed of the $n$ last daily returns. We intentionally assumes $n$ last returns to emphasize that the choice of taking $n=5$ (last week) or $n=10$ (last two weeks) or $n=20$ (last month) is a model design decision. 
As returns are continuous, our state space is $\mathbb{R}_{+}^n$, which by RL standard is very large. 
Our possible actions are each day threefold: either do nothing, buy or sell the Facebook stock. 
If we decide to enter in a new position at time $t$, this will only be materialized the next day and hence 
we will be initially having an open position only at time $t+1$ initialized at the entering price $p_{t+1}$. 
Hence, if we only keep the position for one period, we will be facing the return $r_{t+2}$ as our position
 will be only closed at time $t+2$.

As for the reward, we take the Sharpe ratio of the trading strategy. As we compute the Sharpe ratio with daily returns, to compute the annual Sharpe ratio, we multiply the daily Sharpe ratio by a scaling factor equal to $\sqrt{250}$, assuming 250 trading days per year. We compute the daily Sharpe ratio as the mean of daily returns over their standard deviations. This implies in particular that we implicitly take a benchmark rate in the Sharpe ratio equal to 0.

To compute our Mark to Market (the value of our trading strategy), we mark any open position to the last know price. 
We parametrize our policy with a deep network consisting of two fully connected layers composed of 16 ReLU nodes. 
We keep the layer size as an hyper parameter and use initially REINFORCE and then A2C method to solve this reinforcement learning problem.

We provide the overall achieved Sharpe ratio with the Sharpe ratio reward choice for various learning rate and annealing rate. 
We obtained an overall Sharpe ratio higher than 1 which is quite a nice achievement compared to traditional investment strategies that typically do not perform better than a Sharpe ratio of 1.

\begin{figure}[!htbp]
\centering	
       \includegraphics[width=8cm, height=6cm]{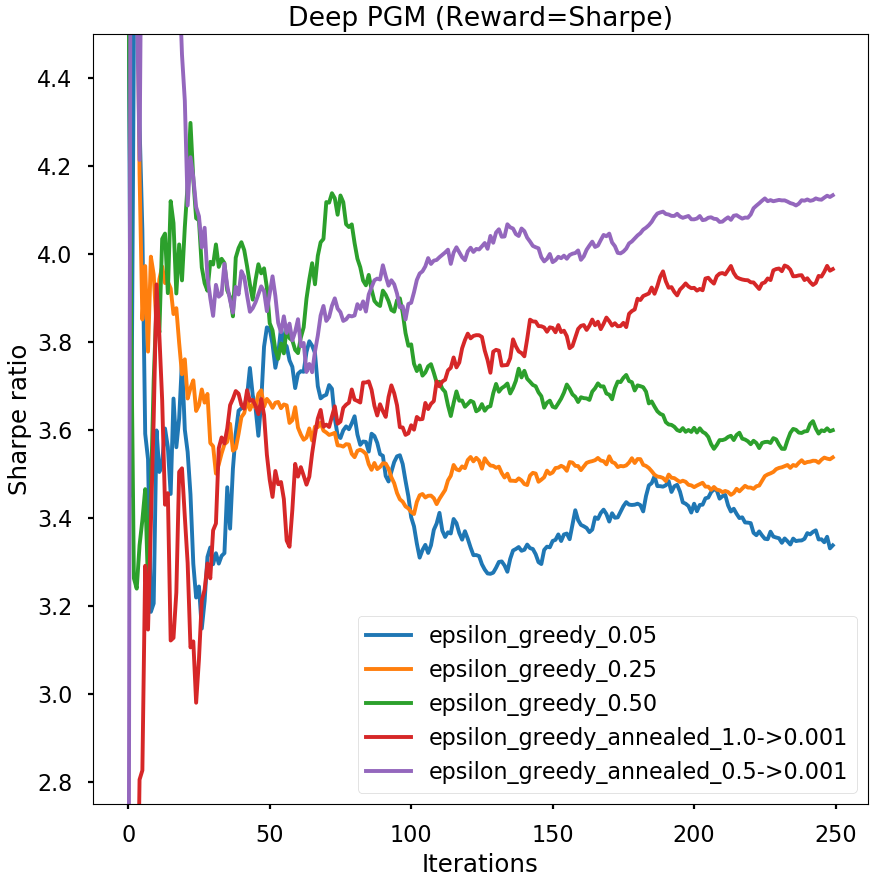}
	\caption{Comparison of various learning rate strategy for our experience. The best strategies are annealing learning rates. These strategies work well as they have the right balance between exploration exploitation. At first, as their learning rate is quite high, they explore more than their fix learning rate counterpart. As we progress in the algorithm, the learning rate decreases and this strategy progressively shift from exploration to exploitation. }
	\label{fig:exp1}
\end{figure}

We can notice in figure \ref{fig:exp1} that the overall Sharpe ratio is above 3 which is very high by financial market standards. This could be explained by the fact that the Facebook stock has been incredibly raising over the last five years. Hence the algorithm has not much difficulty finding the optimal strategy that is to buy and hold the stock
 
\section{Conclusion}\label{sec:conclusion}
We show in this article that there are tight connections between SL and RL. PGM in RL can be cast as cross entropy minimization problems where true labels are replaced by expected future reward or advantage while the log term is changed into the log policy term. This analogy takes its root from the minimization problem where we are looking for the parameters that maximizes the expected futures reward or advantage. Should this optimization objective changed, we conjecture that we could make other analogies between SL and RL

\bibliography{biblio}

\section{Supplementary materials}
\subsection{Proof of proposition \ref{prop:prop1}\label{proof:prop1} }
We provide here a quick proof of REINFORCE with modern notations. 
Let us denote by  $r(s_t, a_t) $ the reward for a state $s_t$ and action $a_t$. Let us assume we have some time horizon (that may be either finite or infinite).  Let us denote by $\tau = (s_1, a_1, \ldots, s_T, a_T) $ a trajectory generated by our policy approximator governed by a parameter $\theta$. Using the Markov property of our MDP process, the probability of a given trajectory $\mathbb{P}(\tau | \theta)$ can be decomposed into a product of conditional probabilities as follows:
\begin{equation}\label{eq:traj_proba}
\mathbb{P}(\tau | \theta) =  \mathbb{P}(s_1) \prod_{t=1}^T \pi_\theta(a_t \mid s_t) \mathbb{P}(s_{t+1} \mid s_t, a_t)
\end{equation}
Taking the log of the above equation \eqref{eq:traj_proba}, using the fact that the log of a product is the sum of the logs, we get:
\begin{eqnarray}\label{eq:traj_proba2}
\log \mathbb{P}(\tau | \theta)  &=& \log \mathbb{P}(s_1) + \sum_{t=1}^T \log \pi_\theta(a_t \mid s_t) \nonumber \\
&& \qquad + \log \mathbb{P}(s_{t+1} \mid s_t, a_t)
\end{eqnarray}
Differentiating with respect to theta gives (since most of terms do not depend on $\theta$):
\begin{eqnarray}\label{eq:gradient}
\nabla_{\theta} \log \mathbb{P}(\tau | \theta)  &=& \sum_{t=1}^T \nabla_{\theta} \log \pi_\theta(a_t \mid s_t) 
\end{eqnarray}

Recall we want to minimize the expected return, $J(\theta)$, defined as 
\begin{eqnarray}
J(\theta) = \mathbb{E}_{\tau \sim \mathbb{P}(\tau | \theta)} \left[ \sum_{t=1}^T r(s_t, a_t) \right] = \int_{\tau} R(\tau)  \mathbb{P}(\tau | \theta) d\tau
\end{eqnarray}

The above notation $\tau \sim \mathbb{P}(\tau | \theta)$ indicates that we're sampling trajectories $\tau$ from the probability distribution of our policy approximator governed by $\theta$ and $R(\tau)$ is the sum of all the future (discounted) rewards.

To find the optimal $\theta$, we do gradient descent and hence need to compute
\begin{eqnarray}
\nabla_{\theta} J(\theta) &=& \nabla_{\theta} \int_{\tau} R(\tau)  \mathbb{P}(\tau | \theta) d\tau \nonumber 
\end{eqnarray}
Using the fact that we can interchange integral and expectation \eqref{eq:interchange}, assuming smooth functions and Lebesgue dominated convergence to justify that we can bring the gradient under the integral, we get:
\begin{eqnarray}
\nabla_{\theta} J(\theta) &=& \int_{\tau}   R(\tau)  \nabla_{\theta} \mathbb{P}(\tau | \theta) d\tau \label{eq:interchange}
\end{eqnarray}
As the log gradient of a function is the quotient of the gradient and the function, we have:
\begin{eqnarray}
\nabla_{\theta} J(\theta) &=& \int_{\tau}   R(\tau)  \nabla_{\theta} \log \mathbb{P}(\tau | \theta) \mathbb{P}(\tau | \theta)  d\tau \label{eq:log_gradient}
\end{eqnarray}
Expressed the integral as an expectation, we conclude:
\begin{eqnarray}
\nabla_{\theta} J(\theta) &=& \mathbb{E}\left[  R(\tau) \nabla_{\theta} \log \mathbb{P}(\tau | \theta) \right] \label{eq:expectation}
\end{eqnarray}
Finally, using the fact that the gradient of the log probability of the trajectory is the sum of the gradient of the log policy probabilities \eqref{eq:gradient}, we obtain the final expression:
\begin{eqnarray}
\nabla_\theta J(\theta) = \mathbb{E}
\left[  R(\tau)  \sum_{t=1}^T \nabla_\theta \log{\pi_\theta}(a_t \mid s_t)  \right]
\end{eqnarray}

To turn this into something tractable, just express this as a Monte Carlo sum as follows:
\begin{eqnarray}
\nabla_\theta J(\theta) = \lim_{N \to \infty} \frac{1}{N} \sum_{i=1}^N  \sum_{t=1}^T R_{i}(\tau)   \nabla_\theta \log{\pi_\theta}(a_{i, t} \mid s_{i,t})
\end{eqnarray}
As the RHS does only depend on $\theta$ in the term $ \log{\pi_\theta}(a_{i, t} \mid s_{i,t})$, minimizing $ J(\theta) $ is the same as minimizing
$\widetilde{J}(\theta)$ given by:
\begin{eqnarray}
\widetilde{J}(\theta) = \lim_{N \to \infty} \frac{1}{N} \sum_{i=1}^N  \sum_{t=1}^T R_{i}(\tau)  \log{\pi_\theta}(a_{i, t} \mid s_{i,t}) 
\end{eqnarray}

As for Advantage Actor Critic, the above proof is the same and leads to the result.
\qed

\end{document}